\RequirePackage{fix-cm}
\documentclass[smallextended]{svjour3}       
\smartqed  
\usepackage[defaultcolor=red]{changes}
\usepackage{graphicx}
\usepackage{amsmath}
\usepackage{amssymb}

\usepackage{amsthm}
\usepackage[colorlinks=true]{hyperref}

\usepackage{subfigure}
\usepackage{paralist}
\usepackage[misc]{ifsym}
\begin{document}
	
	\title{Approximation of Nonlinear Functionals Using Deep ReLU Networks
	}
	
	
	\author{Linhao Song \and Jun Fan \and Di-Rong Chen \and Ding-Xuan Zhou 
	}
	
	
	\institute{Linhao Song \at
		School of Mathematical Science, Beihang University, Beijing, China
		\at School of Data Science, City University of Hong Kong, Kowloon, Hong Kong\\
		\email{linhasong2-c@my.cityu.edu.hk}           
		\and
		Jun Fan \Letter \at Department of Mathematics, Hong Kong Baptist University, Kowloon, Hong Kong\\
		\email{junfan@hkbu.edu.hk}
		\and
		Di-Rong Chen \at School of Mathematical Science, Beihang University, Beijing, China\\
		\email{drchen@buaa.edu.cn}
		\and
		Ding-Xuan Zhou \at School of Mathematics and Statistics, University of Sydney, Sydney NSW 2006, Australia\\
		\email{dingxuan.zhou@sydney.edu.au}
	}
	
	\date{Received: date / Accepted: date}

	\maketitle
	
	\begin{abstract}
		In recent years, functional neural networks have been proposed and studied in order to approximate nonlinear continuous functionals defined on $L^p([-1, 1]^s)$ for integers $s\ge1$ and $1\le p<\infty$. However, their theoretical properties are largely unknown beyond universality of approximation or the existing analysis does not apply to the rectified linear unit (ReLU) activation function. To fill in this void, we investigate here the approximation power of functional deep neural networks associated with the ReLU activation function by constructing a continuous piecewise linear interpolation under a simple triangulation. In addition, we establish rates of approximation of the proposed functional deep ReLU networks under mild regularity conditions. Finally, our study may also shed some light on the understanding of functional data learning algorithms.
	\end{abstract}
	
	\section{Introduction}
	\label{intro}
	The past decade has witnessed unquestionable success of deep learning based on deep neural networks in artificial intelligence. The invention of neural networks was originally inspired by neuron activities in human brains which dates back to the 1940s. Triggered by the availability of big data and the advance in computing power, deep neural networks have become prevalent in various fields of science, business, industry and many others. As is well known, neural networks effectively implement nonlinear mappings approximating functions that are learned based on a set of input-output data, typically through stochastic gradient descent (SGD). In spite of their impressive performance, a more thorough theoretical understanding of why they work so well is still highly demanded.
	
	The building blocks of a neural network are processing units. When an input vector $\mathbf{x}\in\mathbb{R}^d$ is fed into the network, a processing unit computes the function $\sigma(w\cdot \mathbf{x}+b)$, where $\sigma:\mathbb{R}\rightarrow \mathbb{R}$ is called an activation function, $w\in\mathbb{R}^d$ and $b\in\mathbb{R}$ are called weight vector and threshold respectively. The theory of function approximation by shallow or multi-layer neural networks was well developed around 1990. In \cite{Leshno93,Hornik89,Cybenko} much effort has been dedicated to understand the universality of this approximation for any non-polynomial locally bounded and piecewise continuous activation function, which was recently developed for deep convolutional neural networks with ReLU in \cite{Zhou18,Zhou20}. Besides, quantitative results about rates of approximation were also obtained
	in \cite{Hornik89,Barron93,Mhaskar93,Chui96}. However, most existing results in the literature about rates of approximation by neural networks were developed for infinitely differentiable activation functions $\sigma$ with one of the following two special assumptions: one is that for some $u_0\in\mathbb{R}$, \begin{equation} \label{assump1}
		\sigma^{(k)}(u_0)\neq 0, \quad\forall k\in\mathbb{N}_0,
	\end{equation}
	where $\sigma^{(k)}$ denotes the $k$-th order derivative of $\sigma$. The other assumption is that for some integer $i\neq 1,$ there holds
	\begin{equation} \label{assump2}
		\lim\limits_{u\rightarrow -\infty}\frac{\sigma(u)}{|u|^i}=0, \text{ and } \lim\limits_{u\rightarrow\infty}\frac{\sigma(u)}{u^i}=1.
	\end{equation}
	In modern deep learning models, the rectified linear unit (ReLU) is the most commonly used activation function due to its ease of computation and resistance to gradient vanishing. The ReLU activation function is defined by
	\begin{equation*}
		\sigma(u)=\max\{u,0\},
	\end{equation*}
	which is a piecewise linear function and does not satisfy the assumptions \eqref{assump1} or \eqref{assump2}. Recently, explicit rates of approximation by ReLU networks were obtained in \cite{Klusowski18} for shallow nets, in \cite{Shaham18} for nets with 3 hidden layers, and in \cite{Yarotsky17,Telgarsky16,Petersen18} for nets with more layers. Moreover, \cite{Yarotsky17} shows that deep ReLU networks are more efficient than shallow ones in approximating smooth functions and derives upper and lower bounds for the neural network complexity of approximation in Sobolev spaces.

	With the rapid growth of modern technology, learning with infinite dimensional data (referred as functional data) has become an important and challenging task in machine learning since the pioneering work \cite{Ramsay97}. Traditional functional data analysis based on kernel methods and functional principal component analysis mainly focuses on the estimation of linear target functionals \cite{Chen2022}, which is usually not true in practice. This motivates us to consider using neural networks to approximate nonlinear functionals defined on the infinite dimensional input space $L^p([-1, 1]^s)$ with $1\le p<\infty$. In \cite{Stinchcombe99}, one type of generalized neural networks is defined, where the input space $\mathcal{X}$ is not limited to $\mathbb{R}^d$ but can be any locally convex topological vector space. It shows that if the activation function $\sigma$ guarantees that the classical shallow networks defined on $\mathbb{R}^d$ are universal, then the proposed shallow generalized networks are also universal. In \cite{Rossi05}, the concept of functional multi-layer perceptrons is introduced by letting $\mathcal{X}=L^p(\mathbb{R}^s)$. The consistency of the proposed method is obtained by adopting the results of generalized networks in \cite{Stinchcombe99}. Throughout the paper, we refer the functional multi-layer perceptrons defined in \cite{Rossi05} as functional neural networks. To avoid learning functions in the functional neural networks, \cite{Rossi05} also proposes so-called parametric functional neural networks, which will be defined in Section 2. In \cite{Mhaskar97}, Mhaskar shows that if the activation function satisfies the assumption \eqref{assump1}, then any continuous functional on a compact domain can be approximated with any precision by a shallow parametric functional neural networks with sufficient width. Both upper and lower bounds on the rates of approximation are provided in terms of network complexity in \cite{Mhaskar97}.
	
	In this paper, we propose a parametric functional neural network with ReLU activation function aiming at approximating nonlinear continuous functional defined on $L^p([-1, 1]^s)$. First, we map the infinite dimensional domain $L^p([-1, 1]^s)$ into a finite dimensional polynomial space such that the original problem suffices to the approximation of multivariate functions. Then we construct a piecewise linear interpolation under a simple triangulation, not surprising, which is also a deep ReLU network we need. At last, we show that the same rate of approximation as in \cite{Mhaskar97} can be achieved in terms of the number of nonzero parameters in the proposed neural network under mild regularity conditions.
	
	The rest of this paper is structured as follows. Section 2 introduces the definitions for several types of neural networks. Section 3 is concerned with some notations, statement of assumptions and our main results.  Section 4 presents two important propositions and gives proofs of the main theorems. The paper concludes in Section 5 and some lemmas that are used in the proofs can be found in Appendix.
	
	\section{Definition of functional deep neural networks} Deep neural networks involve the choice of an activation function $\sigma:\mathbb{R}\rightarrow\mathbb{R}$ and a network architecture. In this paper we focus on ReLU.
	For $\mathbf{b}=(b_1,\cdots,b_d)\in\mathbb{R}^d$, we define the shifted activation function $\sigma_{\mathbf{b}}:\mathbb{R}^d\rightarrow\mathbb{R}^d$ as
	$$\sigma_\mathbf{b}\left( \begin{matrix}
		
		x_1 \\
		\vdots\\
		x_d
		
	\end{matrix}\right)=\left( \begin{matrix}
		
		\sigma(x_1+b_1) \\
		\vdots\\
		\sigma(x_d+b_d)
		
	\end{matrix}\right).$$
	
	A network architecture $(J,\mathbf{d})$ consists of a positive integer $J$ which is the number of hidden layers and width vector $\mathbf{d}=(d_1,\cdots,d_J)\in \mathbb{N}^J$ which indicates the width in each hidden layer. 
	
	Throughout the paper, we use $(\cdot)'$ to denote the transpose of a vector $(\cdot)$. Denote by $|\cdot|_p$ the vector $p$-norm, that is, $|a|_p=(\sum_{i=1}^{d}|a_i|^p)^{\frac{1}{p}}$ if $a$ is a vector with $d$ components.
	We first introduce the definition of a neural network for approximating multivariate functions.
	\begin{definition} [Classical net] \label{def 1}
		A multi-layer fully connected neural network $H:\mathbb{R}^{d_0}\rightarrow \mathbb{R}$ with network architecture $(J,\mathbf{d})$ is any function that takes the form
		\begin{equation} \label{classical}
			H(\mathbf{x})=\mathbf{a}'\sigma_{\mathbf{b_J}}W_J\sigma_{\mathbf{b}_{J-1}}W_{J-1}\cdots\sigma_{\mathbf{b}_2}W_2\sigma_{\mathbf{b}_1}W_1\mathbf{x},
		\end{equation}
		where $\mathbf{x}\in\mathbb{R}^{d_0}$, $\mathbf{a}\in\mathbb{R}^{d_J},\mathbf{b}_j\in\mathbb{R}^{d_j}$, and $W_j=\left(W_j^{i,k}\right)$ is a $d_j\times d_{j-1}$ weight matrix, $j=1,\cdot\cdot\cdot,J.$
	\end{definition}
	
	Let $s$ be a positive integer, we consider the function space $L^p([-1,1]^s)=\{f:[-1,1]^s\rightarrow\mathbb{R} \ |\ f \text{ is measurable and }||f||_p< \infty\}$ where
	\begin{equation*}
		||f||_p=\left(\int_{[-1,1]^s}|f(\mathbf{x})|^pd\mathbf{x}\right)^{\frac{1}{p}}, \hbox{ when } 1\le p <\infty
	\end{equation*}
	and
	\begin{equation*}
		||f||_p=\mathop{ess\sup}\limits_{\mathbf{x}\in[-1,1]^s}|f(\mathbf{x})|, \hbox{ when } p=\infty.
	\end{equation*}
	Specially when $p=2$, $L^2([-1,1]^s)$ is a Hilbert space, and we denote the inner product by $\langle \cdot,\cdot\rangle$, that is,
	\begin{equation}\label{innnerproduct}
		\langle f_1,f_2\rangle = \int_{[-1,1]^s}f_1(\mathbf{x})f_2(\mathbf{x})d\mathbf{x}.
	\end{equation}
	
	We now introduce the definition of functional neural network \cite{Rossi05} for approximating functional defined on $L^p([-1,1]^s)$.
	\begin{definition} [Functional net] \label{fn}
		A functional neural network $\Theta: L^p([-1,1]^s)\rightarrow \mathbb{R}$ with network architecture $(J,\mathbf{d})$ is any functional that takes the form
		\begin{equation} \label{functional}
			\Theta(f)=\mathbf{a}'\sigma_{\mathbf{b_J}}W_J\sigma_{\mathbf{b}_{J-1}}W_{J-1}\cdots\sigma_{\mathbf{b}_2}W_2\sigma_{\mathbf{b}_1}T(f),
		\end{equation}
		where $f\in L^p([-1,1]^s),$ $\mathbf{a}\in\mathbb{R}^{d_J},\mathbf{b}_j\in\mathbb{R}^{d_j}, j=1,\cdots,J.$ $W_j=\left(W_j^{i,k}\right)$ is a $d_j\times d_{j-1}$ matrix, $j=2,\cdot\cdot\cdot,J.$ Here $T:L^p([-1,1]^s)\rightarrow\mathbb{R}^{d_1}$ is a bounded linear operator with $T(f)=\big(\int_{[-1,1]^s} f(\mathbf{x})g_1(\mathbf{x})d\mathbf{x},\cdots,\int_{[-1,1]^s}f(\mathbf{x})g_{d_1}(\mathbf{x})d\mathbf{x} \big)'$ for $\{g_k\}_{k=1}^{d_1}\in L^q([-1,1]^s)$, and $q$ is the conjugate exponent of $p$ satisfying $1/p+1/q=1$.
	\end{definition}

	\begin{remark}
		Comparing \eqref{classical} and \eqref{functional}, the difference is that the functional net uses a bounded linear operator $T$ in the first hidden layer (functional layer), while the classical net uses a numerical weight matrix $W_1$ (numerical layer).  An example of functional net with network architecture $(J=2,\mathbf{d}=(3,2))$ is given in Figure \ref{FNN}.
	\end{remark}
	
	\begin{remark}
		The dual space of $L^{\infty}([-1,1]^s)$ is larger than $L^1([-1,1]^s)$. In this case, we can restrict the operator $T$ to be induced by functions in $L^1([-1,1]^s)$.
	\end{remark}
	
	Note that one drawback of this functional net is that the functions $\{g_k\}_{k=1}^{d_1}$ can not be learned directly. This can be addressed by using parametric representation of functions as follows.
	
	\begin{definition} [Parametric functional net] \label{pfn}
		Let $d_0\in\mathbb{N}$, $\{v_k\}_{k=1}^{d_0}$ be a linearly independent set in $L^q([-1,1]^s)$, and denote 
		\begin{equation*}
			\nu_i=\int_{[-1,1]^s} f(\mathbf{x})v_i(\mathbf{x})d\mathbf{x},\quad i=1,\cdots,d_0,
		\end{equation*}
		then a parametric functional neural network $\Theta_v: L^p([-1,1]^s)\rightarrow \mathbb{R}$ with network architecture $(J,\mathbf{d})$ is any functional that takes the form
		\begin{equation} \label{parametric}
			\begin{aligned}
				\Theta_v(f)&=\mathbf{a}'\sigma_{\mathbf{b_J}}W_J\sigma_{\mathbf{b}_{J-1}}W_{J-1}\cdots\sigma_{\mathbf{b}_2}W_2\sigma_{\mathbf{b}_1}W_1\mathbf{\nu},\\
				\mathbf{\nu}&=(\nu_1,\cdots,\nu_{d_0})',
			\end{aligned}
		\end{equation}
		where $f\in L^p([-1,1]^s),$ $\mathbf{a}\in\mathbb{R}^{d_J},\mathbf{b}_j\in\mathbb{R}^{d_j}$, and $W_j=\left(W_j^{i,k}\right)$ is a $d_j\times d_{j-1}$ matrix, $j=1,\cdot\cdot\cdot,J.$ Here the subscript $v$ is used to indicate that the linearly independent set in $L^p([-1,1]^s)$ for parametrization is $\{v_k\}_{k=1}^{d_0}.$
	\end{definition}
	
	\begin{remark}
		Here $\{v_k\}_{k=1}^{d_0}$ is a set of known functions that does not need to be learned, and the choice of them is related to a continuous linear operator $V_m$ to be defined in Subsection \ref{Discretizing}.
	\end{remark}
	
	\begin{remark}
		A functional net is also a parametric functional net if we assume that $\{g_i\}_{i=1}^{d_1}$ in Definition \ref{fn} has a parametric representation using the linearly independent set $\{v_k\}_{k=1}^{d_0}$, that is,
		\begin{equation*}
			g_i=W_1^{i,1}v_1+W_1^{i,2}v_2+\cdot\cdot\cdot+W_1^{i,d_0}v_{d_0},
		\end{equation*}
		for some coefficients $W_1^{i,k}\in\mathbb{R}$, $k=1,\cdot\cdot\cdot,d_0$, $i=1,\cdot\cdot\cdot,d_1$. Then the problem of learning functions $\{g_i\}_{i=1}^{d_1}$ turns into learning weight matrix $W_1$, and this is the reason why \eqref{parametric} is called parametric functional net. 
	\end{remark}

	In addition to the network architecture $(J,\mathbf{d})$, the network \eqref{parametric} is also determined by the numerical weight matrix $W_j$, shift vectors $\mathbf{b}_j$, ~$j=1,\cdot\cdot\cdot,J$~, and output vector $\mathbf{a}$. We denote by $M(\Theta_v)=M:=\sum_{j=1}^{J}||W_j||_0+\sum_{j=1}^{J}||\mathbf{b}_j||_0+||\mathbf{a}||_0$ the total number of nonzero weights of ~$\Theta_v$, where $||\cdot||_0$ means the number of nonzero entries in a vector or a matrix. We will use $M$ in this paper as the complexity of the neural networks to characterize rates of approximation.
	
	\begin{figure}[htbp]
		{
			\centering
			\includegraphics[width=0.7\textwidth]{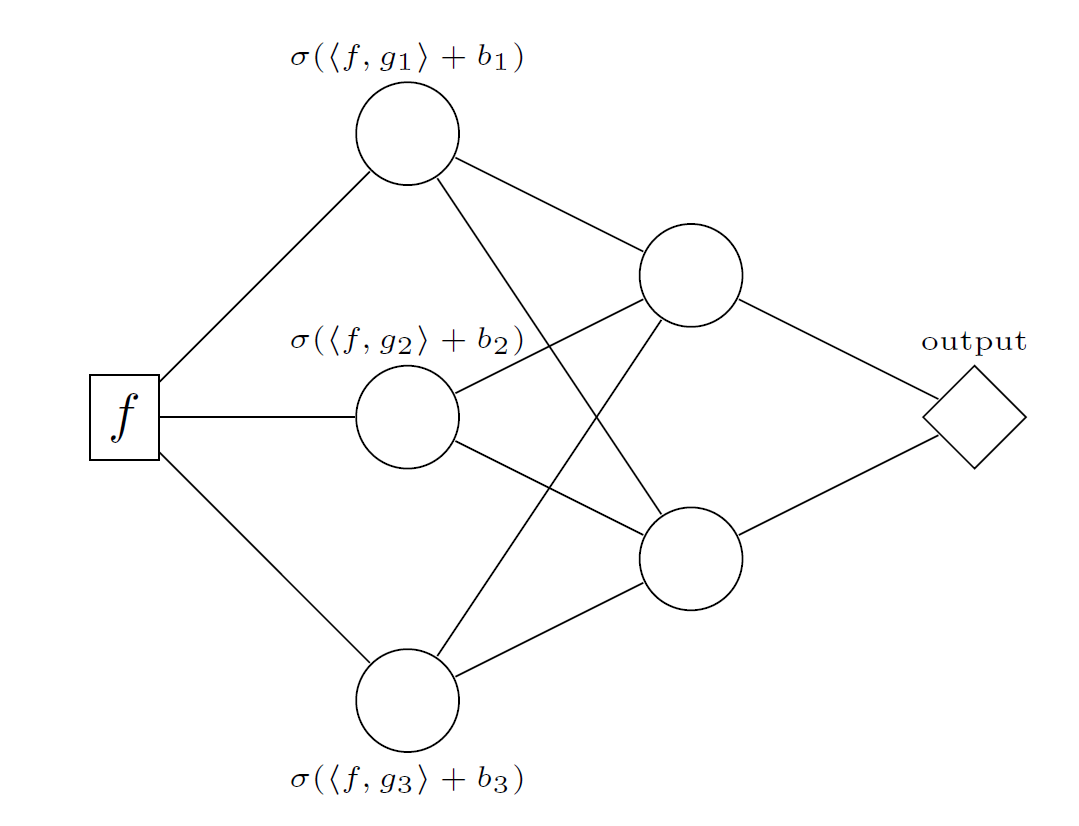}
			\caption{A functional net with 2 hidden layers. The input neuron (rectangle) is a function, the rest 5 hidden neurons (circle) and the output neuron (diamond) are real numbers. In the first hidden layer, three functional weights $g_1$, $g_2$, $g_3$ are used.}
			\label{FNN}
		}
	\end{figure}

	\section{Main results on rates of approximation}
	
	In this section, we state our main results and the proof will be given in Section 4. Let $F:L^p([-1,1]^s)\rightarrow\mathbb{R}$ be the target functional. We are interested in approximating $F$ by constructing a parametric functional net. 
	
	\subsection{Assumptions on input function and target functional}
	When deriving the quantitative results about rates of approximation for a function defined on $\mathbb{R}^d$, we need to make priori assumptions on its smoothness. For a target functional, we can make a similar assumption by adopting the definition of modulus of continuity. We assume that the target functional $F:L^p([-1,1]^s)\rightarrow\mathbb{R}$, though unknown, is continuous with modulus of continuity $\omega_F:(0,\infty)\rightarrow(0,\infty)$ given by
	\begin{equation*}
		\omega_F(r)=\sup\big\{|F(f_1)-F(f_2)|:f_1,f_2\in L^p([-1,1]^s), ||f_1-f_2||_p\le r\big\}.
	\end{equation*}
	It is well known that the modulus of continuity $\omega_F$ is an increasing function and satisfies the following property
	\begin{equation*}
		{
			|F(f_1)-F(f_2)|\le\omega_F\left(||f_1-f_2||_p\right), \quad \forall f_1,f_2\in L^p([-1,1]^s).
		}
	\end{equation*}
	Moreover, the sub-additive property holds, that is,
	\begin{equation*}
		{
			\omega_F(r_1+r_2)\le\omega_F(r_1)+\omega_F(r_2), \quad r_1,r_2>0.
		}
	\end{equation*}

	\subsection{Properties of compact subsets of $L^p([-1,1]^s)$}
	We make a priori assumption that the input function belongs to a compact subset $K$ of $L^p([-1,1]^s)$. Under this assumption, according to \cite[page 33]{Lorentz53}, there exists a constant $c_K$ such that
	\begin{equation} \label{f}
		||f||_p\le c_K, \quad\forall f\in K,
	\end{equation}
	and the approximation by polynomials in $\Pi_{m}:=\Pi_m([-1,1]^s)$, the class of all polynomials in $s$ variables
	of coordinatewise degree not exceeding $m$, can be bounded as
	\begin{equation} \label{epsilon}
		\min\limits_{Q\in\Pi_m}||f-Q||_p \le \epsilon_{m,K},\quad\forall f\in K,
	\end{equation}
	where $\big\{\epsilon_{m,K}\big\}_{m=1}^{\infty}$ is a sequence converging to $0$ which depends only on $K$, meaning that the convergence is uniformly on $f\in K$.
	
	Now we are in the position to state our main results.
	\begin{theorem} \label{result1}
		Let $s,m,M\in\mathbb{N}$, $1\le p<\infty$, and set $d_0=(2m+1)^s.$ 
		If $F:L^p([-1,1]^s)\rightarrow\mathbb{R}$ is a continuous functional with modulus of continuity $\omega_F$, then for any compact set $K\subset L^p([-1,1]^s)$, there exists a parametric functional deep ReLU network $\Theta_v$ with the depth $J=d_0^2+d_0+1$ and the number of nonzero weights ${M}$ such that
		\begin{equation*}
			\sup_{f\in K}|F(f)-\Theta_v(f)|\le \omega_F(C\epsilon_{m,K}) +2(2m+1)^s\omega_{F}\left(\frac{Cd_0^{\frac{4}{d_0}}m^{\theta}}{M^{\frac{1}{d_0}}}\right),
		\end{equation*}
		where $\theta=2s|\frac{1}{p}-\frac{1}{2}|$, and $C$ is a constant independent of $m$ or $M$, which will be given explicitly in the proof.
	\end{theorem}
	
	We give two examples to illustrate \eqref{epsilon} in the following remarks.
	\begin{remark}
		Let	$k\ge 1$ be an integer. We consider the Sobolev space $W^{k,p}([-1,1]^s)$, which consists of functions whose partial derivatives of order up to $k$ belong to $L^p([-1,1]^s)$. The Sobolve norm of $f\in W^{k,p}([-1,1]^s)$ is defined by
		\begin{equation*}
			||f||_{W^{k,p}}:=\sum_{0\le \alpha\le k}||{\mathop{D}}^{\alpha}f||_p,
		\end{equation*}
		where for multi-integer $\alpha=(\alpha_1,\cdots,\alpha_s)\in \mathbb{Z}^s$, $0\le \alpha\le k$ {means each entry of $\alpha$ is an integer between $0$ and $k$.} Let $|\alpha|=\sum_{j=1}^s|\alpha_j|$ and
		\begin{equation*}
			{\mathop{D}}^{\alpha}f=\frac{\partial^{|\alpha|}f}{\partial x_1^{\alpha_1}\cdots x_s^{\alpha_s}}, \quad \mathbf{\alpha}\ge 0.
		\end{equation*}
		It is well-known \cite{Mhaskar96} that there exists a constant $c_{s,k,p}$ such that
		\begin{equation*}
			\min\limits_{Q\in\Pi_m}||f-Q||_p \le c_{s,k,p}m^{-k}||f||_{W^{k,p}}, \quad \forall m\ge 0.
		\end{equation*}
		If we set $K$ to be the unit ball of $W^{k,p}([-1,1]^s)$, then $\epsilon_{m,K}=c_{s,k,p}m^{-k}$.
	\end{remark}
	
	\begin{remark} \label{remark}
		Take into account the set of functions satisfying a H\"older condition of order ~$\beta>0,$ denoted by $C^{\beta}([-1,1]^s)$.
		For $0<\beta\le 1$, $C^{\beta}([-1,1]^s)$ consists of Lipschitz-$\beta$ functions with norm 
		$$||f||_{C^{\beta}}=||f||_{\infty}+\sup_{x\neq y}\frac{|f(x)-f(y)|}{|x-y|_2^{\beta}}.$$
		For $\beta=k+\beta'$ with $k\in\mathbb{N}$ and $0<\beta'\le 1$, $C^{\beta}([-1,1]^s)$
		consist of $k$ times differentiable functions whose partial derivatives of order $k$ are Lipschitz-$\beta'$ functions with an equivalent norm
		\begin{equation*}
			||f||_{C^{\beta}}:=\sum_{|\alpha|\le k}||{\mathop{D}}^{\alpha}f||_{\infty}+\sum_{|\alpha|= k}||{\mathop{D}}^{\alpha}f||_{C^{\beta'}}.
		\end{equation*}
		It is well-known that for any $\beta >0, s\in\mathbb{N}$, there exists a constant $c_{s,\beta}$ such that
		\begin{equation*}
			\min\limits_{Q\in\Xi_m}||f-Q||_{\infty} \le c_{s,\beta}m^{-\beta}||f||_{C^{\beta}}, \quad \forall m\ge 0,
		\end{equation*}
		where $\Xi_m=\Xi_m([-1,1]^s)$ is the class of all polynomials on $[-1,1]^s$ of degree up to $m$. Setting $K$ to be the unit ball of $C^{\beta}([-1,1]^s)$, then $\epsilon_{m,K}=c_{s,\beta}m^{-\beta}$ as $\Xi_m\subset\Pi_m$.
		
	\end{remark}
	
	\begin{theorem} \label{corro} 
		Let $s,M\in\mathbb{N}$, $1\le p<\infty$, $\beta>0$ and $K$ be the unit ball of $C^{\beta}([-1,1]^s)$. If $\omega_F(r)\le c_6r^{\lambda}$ for some $\lambda\in(0,1]$, then we can find a parametric functional deep ReLU network $\Theta_v$ with depth $$J\le  \widetilde{C}\left(\frac{\log M}{\log(\log M)}\right)^{2}$$ and number of nonzero weights $M$ such that
		\begin{equation}\label{onlyM}
			\sup_{f\in K}|F(f)-\Theta_v(f)|=O\left(\left(\frac{\log M}{\log(\log M)}\right)^{-\frac{\beta\lambda}{s}}\right),
		\end{equation}
		where $c_6$ is a positive constant, and $\widetilde{C}$ is a positive constant depending on $s,\lambda,\beta,p$.
	\end{theorem}
	To approximate the H\"older space $C^{\beta}([-1,1]^s)$, generalized translation networks were used in \cite{Mhaskar97} with infinitely differentiable activation functions satisfying the assumption \eqref{assump1}, and a rate of approximation
	\begin{equation*}
		\omega_F\left(\frac{\log M_p}{\log(\log M_p)}\right)^{-\frac{\beta}{s}}
	\end{equation*}
	was derived in \cite{Mhaskar97}, where $M_p$ denotes the total number of parameters in the translation network. Here, Theorem \ref{corro} is doing the same thing and it reveals that we can still achieve the same rate by using functional deep ReLU networks if the modulus of continuity $\omega_F$ satisfies some condition. Also, by using the nonlinear $N$-width for the set of functionals with a common modulus of continuity $\omega_F$ in the case of certain compact $K\subset L^2([-1,1]^s)$, \cite{Mhaskar97} further established a lower bound 
	$$\omega_F\left((\log M_p)^{-\frac{\beta}{s}}\right).$$
	As we can see, the rate given in \eqref{onlyM} matches this lower bound up to the $\log(\log(M_p))$ term in the denominator.

	\section{Approximation by continuous linear operators and deep ReLU networks}
	In this section, we introduce two important propositions and then use them to prove Theorems $\ref{result1}$ and \ref{corro}.
	\subsection{Discretizing functions into vectors} \label{Discretizing}
	It is well-known in approximation theory that there exists a continuous linear operator $V_m: L^p([-1,1]^s)\rightarrow\Pi_{2m}$ such that
	\begin{equation} \label{Vm}
		||f-V_mf||_p\le c\min\limits_{Q\in\Pi_m}||f-Q||_p, \quad \forall f \in L^p([-1,1]^s),
	\end{equation}
	where $c$ is a positive constant depending only on $p$ and $s$. One example of operators satisfying \eqref{Vm} is mentioned in \cite{Mhaskar96}, the construction of which is based on the de la Vall\'ee Poussin operator. Actually, there are many such operators known in the literature \cite{Timan63,Lorentz66}.
	
	Due to the continuity, linearity and finite range of $V_m$, we can represent it in an explicit way. For simplicity here we consider Legendre polynomials which form a classical orthonormal basis of the function space $L^2([-1,1]^s)$. In the univariate case, Legendre polynomials are defined by
	\begin{equation*}
		\mathbb{L}_n(x):=\frac{(-1)^n\sqrt{n+1/2}}{2^nn!}\left(\frac{d}{dx}\right)^n\{(1-x^2)^n\}, \quad n=0,1,2,\cdots
	\end{equation*} 
	For $\mathbf{x}=(x_1,\cdots,x_s)\in \mathbb{R}^s$, and $\mathbf{k}=(k_1,\cdots,k_s)\in\mathbb{Z}^s_{+}$, we write
	\begin{equation*}
		\mathbb{L}_{\mathbf{k}}(\mathbf{x}):=\prod_{j=1}^{s}\mathbb{L}_{k_j}(x_j).
	\end{equation*}
	Note that the $\{\mathbb{L}_{\mathbf{k}}\}_{\mathbf{k}\in\mathbb{Z}^s_{+}}$ satisfies
	\begin{equation*}
		\langle\mathbb{L}_{\mathbf{k}},\mathbb{L}_{\mathbf{k}'}\rangle=\left\{
		\begin{aligned}
			1,&\qquad \text{if } \mathbf{k}=\mathbf{k}', \\
			0,&\qquad \text{otherwise},
		\end{aligned}
		\right.
	\end{equation*}
	with respect to the inner product $\langle\cdot,\cdot\rangle$ defined in \eqref{innnerproduct}.
	We replace the multi index set $\{0,1,\cdots\}^s$ by the usual one $\{1,2,\cdots\}$ with an order arranged by the total degrees, then the set $\{\mathbb{L}_{\mathbf{k}}(\mathbf{x})\}_{\mathbf{k}\in\mathbb{Z}^s_{+}}$ becomes $\{\mathbb{L}_{k}(\mathbf{x})\}_{k\ge 1}$. It is easy to see that the first $(2m+1)^s$ functions $\{\mathbb{L}_1,\cdots,\mathbb{L}_{(2m+1)^s}\}$ form a basis of the polynomial space $\Pi_{2m}.$ If $1\le p <\infty$ and take $q$ to be the conjugate exponent of $p$, then there exist functions $\{v_k\in L^q([-1,1]^s),k=1,...,(2m+1)^s\}$ depending on $V_m$ such that,
	\begin{equation}\label{vk}
		V_m(f)(\mathbf{x})=\sum\limits_{k=1}^{(2m+1)^s}\left(\int_{[-1,1]^s} f(\mathbf{x})v_k(\mathbf{x})d\mathbf{x}\right)\mathbb{L}_k(\mathbf{x}), \quad \mathbf{x}\in [-1,1]^s.
	\end{equation}
	For simplicity, we use $t=t(m)=(2m+1)^s$ in the rest of this paper, and we choose $\{v_k\}_{k=1}^t$ to be the set used for parametrization in Definition \ref{pfn} in the following theoretical analysis. As for the case $p=\infty$, the form \eqref{vk} dose not hold, since the dual space of $L^{\infty}([-1,1]^s)$ is much larger than $L^1([-1,1]^s)$, therefore $p=\infty$ is not included in this paper. But for a specific family of operators $\{V_m\}$, it is possible to choose $\{v_k\}\subset L^1([-1,1]^s)$ in the representation \eqref{vk}, then the case $p=\infty$ can be covered.

	\begin{proposition} \label{result2}
		Let $s,m\in\mathbb{N}$, $1\le p <\infty$ and set $t=(2m+1)^s$. 
		Take any $\{v_k\}_{k=1}^{t}$ and $V_m$ satisfying \eqref{Vm} and \eqref{vk}. 
		If $F:L^p([-1,1]^s)\rightarrow\mathbb{R}$ is a continuous functional with modulus of continuity $\omega_F$, then for any compact set $K\subset L^p([-1,1]^s)$ and any $f\in K$, we have
		\begin{equation*}
			|F(f)-F(V_mf)|\le \omega_F(c\epsilon_{m,K}).
		\end{equation*}

	\end{proposition}
	\begin{proof}
		By the definition of modulus of continuity, we have
		\begin{equation}\label{moc1}
			|F(f)-F(V_mf)|\le \omega_F(||f-V_mf||_p).
		\end{equation}
		The non-decreasing property of $\omega_F$ together with (\ref{epsilon}) and (\ref{Vm}) lead to
		\begin{equation}\label{moc2}
			\omega_F(||f-V_mf||_p)\le \omega_F(c\epsilon_{m,K}), \quad \forall f\in K.
		\end{equation}
		The desired conclusion follows by combining \eqref{moc1} and \eqref{moc2}.
	\end{proof}
	
	Define an isometric isomorphism {$\phi:\big(\Pi_{2m},||\cdot||_2\big)\rightarrow \big(\mathbb{R}^t,|\cdot|_2\big)$} given by
	\begin{equation*}
		\phi(Q)=\big(\langle Q,\mathbb{L}_1\rangle,\cdot\cdot\cdot,\langle Q,\mathbb{L}_t\rangle\big)'
	\end{equation*}
	for $Q\in\Pi_{2m}$ and denote by $\phi^{-1}$ the inverse of $\phi$. We define $\mu_{m,F}:=F\circ\phi^{-1}:\mathbb{R}^t\rightarrow\mathbb{R}$, then $F(V_mf)=F\circ\phi^{-1}(\phi(V_mf))=\mu_{m,F}(\phi(V_mf))$. The input function $f$ is then discretized to a vector with $t$ components
	\begin{equation} \label{phi}
		\phi(V_mf)=\left(\int_{[-1,1]^s} f(\mathbf{x})v_1(\mathbf{x})d\mathbf{x},\cdot\cdot\cdot,\int_{[-1,1]^s} f(\mathbf{x})v_t(\mathbf{x})d\mathbf{x}\right)'.
	\end{equation}
	
	\subsection{Constructing neural networks for approximation}
	For our analysis, we need a lemma on comparing $L_p$-norms of polynomials which can be found in \cite{Mhaskar97} as Lemma 3.1.
	\begin{lemma} \label{Lp}
		Let  $p,q\in[1,\infty]$, then for any $m\in \mathbb{N}$ and $Q\in \Pi_{2m}$, we have
		\begin{equation*}
			||Q||_p\le c_1m^{2s\max\{\frac{1}{q}-\frac{1}{p},0\}}||Q||_q,
		\end{equation*}
		where $c_1$ is a constant independent of $m$.
	\end{lemma}
	By the isometry of $\phi$, we have
	\begin{equation*}
		|\phi(V_mf)|_{\infty}\le |\phi(V_mf)|_2=||V_mf||_{2}.
	\end{equation*}
	Furthermore, by Lemma \ref{Lp} we have $||V_mf||_2\le c_1m^{2s\max\{\frac{1}{p}-\frac{1}{2},0\}}||V_mf||_p$. We know from \eqref{f}, (\ref{epsilon}) and (\ref{Vm}) that $||V_mf||_p\le ||f-V_mf||_p+||f||_p\le c\epsilon_{m,K}+c_K \le C_K:=\sup_{m\in\mathbb{N}}c\epsilon_{m,K}+c_K$. Denote
	\begin{equation*}
		R=R_{m,K}:=c_1C_Km^{2s\max\{\frac{1}{p}-\frac{1}{2},0\}},
	\end{equation*}
	then $\phi(V_mf)$ falls in the cube $[-R,R]^t$ for all $f\in K.$

	\begin{lemma} \label{omegamu}
		Let $\omega_{\mu_{m,F}}$ be the modulus of continuity of $\mu_{m,F}$, then
		\begin{equation*}
			\omega_{\mu_{m,F}}(r)\le \omega_F(c_1m^{2s\max\{\frac{1}{2}-\frac{1}{p},0\}}r), \quad\forall r>0,
		\end{equation*}
		where $c_1$ is the constant given in Lemma \ref{Lp}.
	\end{lemma}

		\begin{figure}[htbp]
		\centering
		\subfigure[Triangulation in $\mathbb{R}^2$. Each triangle is called a simplex.]{
			\label{spike1}
			\includegraphics[width=0.4\textwidth]{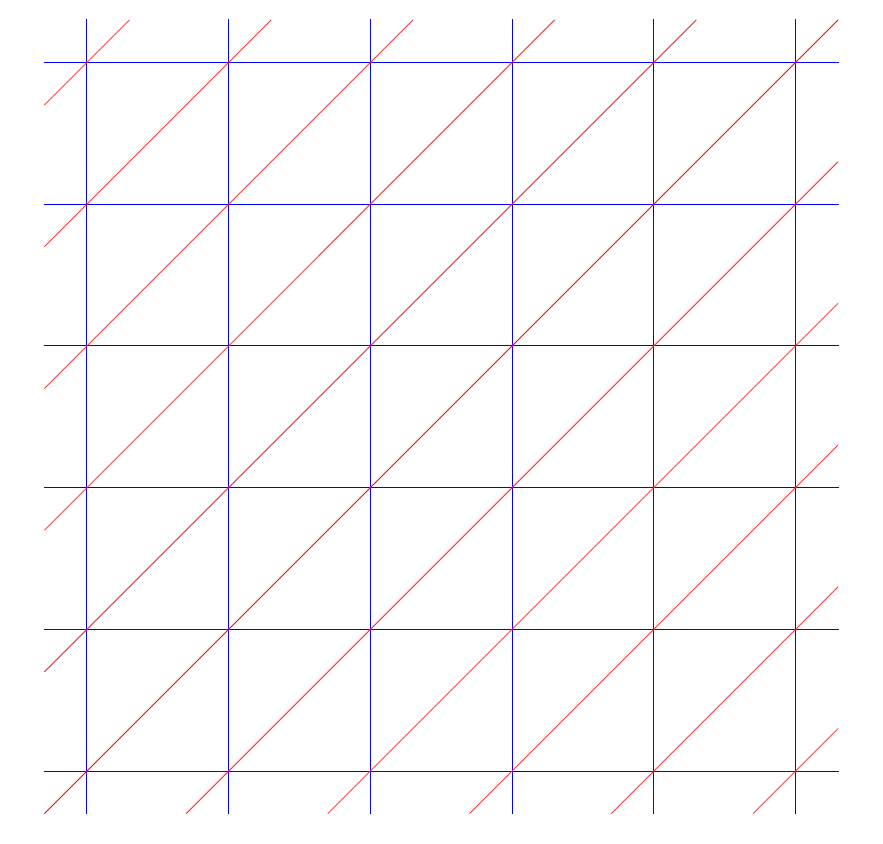}}\qquad
		\subfigure[$S_0$ in $\mathbb{R}^2$. The blue line segment is the boundary of $\partial S_0$, and the coordinate of red dot is $(0,0)$, where the ``spike'' function $\psi$ equals to 1.]{
			\label{spike2}
			\includegraphics[width=0.39\textwidth]{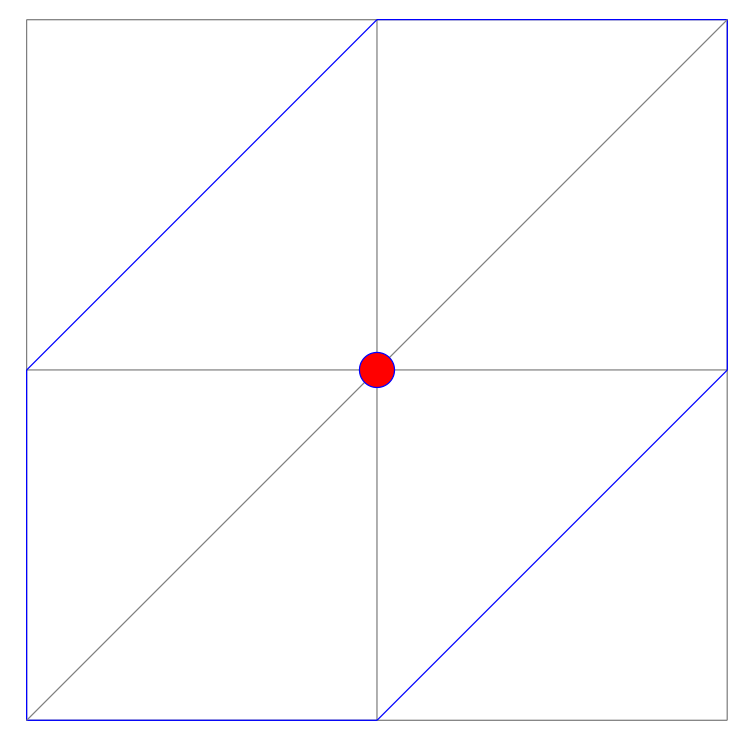}}
		\caption{Triangulation and $S_0$ in $\mathbb{R}^2$.}
		\label{spike}
	\end{figure}
	
	\begin{proof} By the definition of the modulus of continuity of $F$, Lemma \ref{Lp} and isometry of $\phi$, for any $\mathbf{y}_1, \mathbf{y}_2\in\mathbb{R}^t,$ we have
		\begin{equation*}
			\begin{aligned}
				|\mu_{m,F}(\mathbf{y}_1)-\mu_{m,F}(\mathbf{y}_2)|&=|F\left(\phi^{-1}\mathbf{y}_1\right)-F\left(\phi^{-1}\mathbf{y}_2\right)|\\
				&\le\omega_F\left(||\phi^{-1}\mathbf{y}_1-\phi^{-1}\mathbf{y}_2||_p\right)\\
				&\le\omega_F\left(c_1m^{2s\max\{\frac{1}{2}-\frac{1}{p},0\}}||\phi^{-1}(\mathbf{y}_1-\mathbf{y}_2)||_2\right)\\
				&=\omega_F\left(c_1m^{2s\max\{\frac{1}{2}-\frac{1}{p},0\}}|\mathbf{y}_1-\mathbf{y}_2|_2\right),
			\end{aligned}
		\end{equation*}
		which yields the desired conclusion.
	\end{proof}
	
	Once we know the modulus of continuity of $\mu_{m,F}$, we can
		construct a continuous piecewise linear interpolation under a simple triangulation to approximate it.
	
	We denote a simplex in $\mathbb{R}^t$ by
	\begin{equation*}
		\triangle_{\mathbf{0}}=\big\{\mathbf{y}=(y_1,\cdot\cdot\cdot,y_t)\in\mathbb{R}^t:0\le y_{1}\le\cdot\cdot\cdot\le y_{t}\le 1\big\}.
	\end{equation*}
	Then we shift $\triangle_{\mathbf{0}}$ by  $\mathbf{n}=(n_1,\cdot\cdot\cdot,n_t)\in\mathbb{Z}^t$, and permute the coordinates to get a new simplex 
	\begin{equation*}
		\triangle_{\mathbf{n},\rho}=\big\{\mathbf{y}\in\mathbb{R}^t:0\le y_{\rho(1)}-n_{\rho(1)}\le\cdot\cdot\cdot\le y_{\rho(t)}-n_{\rho(t)}\le 1\big\},
	\end{equation*}
	where $\rho\in \mathcal{P}_t$,  the set of all permutations of $t$ elements. According to Lemma \ref{lemma_partition} in Appendix, we know that $\left\{\triangle_{\mathbf{n},\rho}\right\}_{\mathbf{n}\in\mathbb{Z}^t,\rho\in\mathcal{P}_t}$ is a partition of $\mathbb{R}^t$. Therefore, dissecting the whole space $\mathbb{R}^t$ into small simplexes $\left\{\triangle_{\mathbf{n},\rho}\right\}_{\mathbf{n}\in\mathbb{Z}^t,\rho\in\mathcal{P}_t}$ can be viewed as a triangulation.  Figure \ref{spike1} shows this triangulation on $\mathbb{R}^2$.
	
	Denote $\mathbf{0}=(0,...,0)\in\mathbb{R}^t$, and it is easy to see that there exists a unique function $\psi:\mathbb{R}^t\rightarrow\mathbb{R}$ such that
	\begin{enumerate}[(a)]
		\item $\psi(\mathbf{0})=1$, and $\psi(\mathbf{y})=0$ for $\mathbf{y}\in \mathbb{Z}^t\setminus\{\mathbf{0}\}$;
		\item  $\psi$ is continuous in $\mathbb{R}^t$;
		\item $\psi$ is linear in each simplex $\triangle_{\mathbf{n},\rho}$ for $\mathbf{n}\in\mathbb{Z},\rho\in\mathcal{P}_t$.
	\end{enumerate}
	
	Denote by $S_0$ the support of $\psi$. From properties (a), (b) and (c), we know $S_0$ is the union of all simplexes that contains $\mathbf{0}$, that is,
		\begin{equation*}
			S_0=\bigcup_{\mathbf{0}\in\triangle_{\mathbf{n},\rho}}\triangle_{\mathbf{n},\rho}.
		\end{equation*}
		Figure \ref{spike2} shows $S_0$ in $\mathbb{R}^2,$ which contains 6 simplexes. According to Lemma \ref{lemma_convex} in Appendix, we know $S_0$ is a convex set, which implies an explicit representation of $\psi$ as follows by Lemma 3.1 in \cite{He2020}:
		\begin{equation*}
			\psi(\mathbf{y})=\sigma\left(\min_{\triangle\in \mathcal{T}}\{h_{\triangle}(\mathbf{y})\}\right),
		\end{equation*}
		where $\mathcal{T}=\{\triangle_{\mathbf{n},\rho}:\mathbf{0}\in\triangle_{\mathbf{n},\rho} \}$ and $h_{\triangle}$ is the global linear function such that $h_{\triangle}=\psi$ on the simplex $\triangle$. With the help of Lemma \ref{lemma_affine} in Appendix, we know that $h_{\triangle}(\mathbf{y})$ is either of the form $1+y_k-y_j$, $k\neq j$ or $1\pm y_k$,  hence
	\begin{equation}\label{14}
		\psi(\mathbf{y})=\sigma\Big(\min\big\{\min_{k\neq j}(1+y_k-y_j),\min_{k}(1+y_k),\min_{k}(1-y_k)\big\}\Big).
	\end{equation}
	We illustrate the function $\psi$ in $\mathbb{R}^2$ by Figure \ref{spike3}.

	\begin{figure}[htp]
	{
		\centering
		\includegraphics[width=0.6\textwidth]{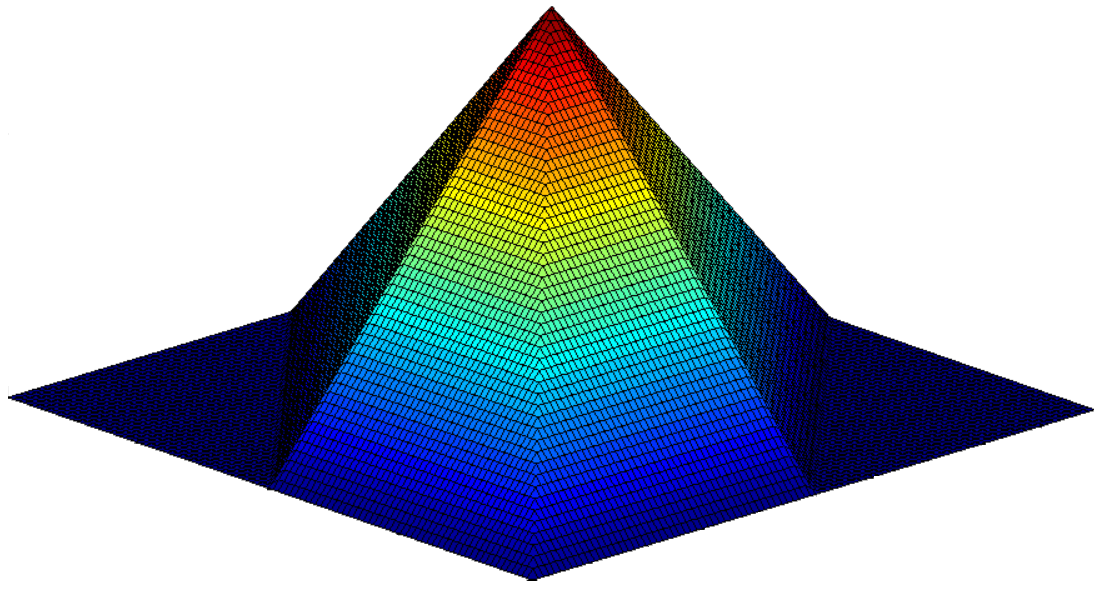}
		\caption{Spike function $\psi$ on $\mathbb{R}^2.$}
		\label{spike3}
	}
\end{figure}

		\begin{lemma} \label{nonzero}
		Let $\mathbf{x}=(x_1,\cdot\cdot\cdot,x_d)\in \mathbb{R}^d$, then function $\min(\mathbf{x})=\min(x_1,\cdot\cdot\cdot,x_d)$ can be seen as a ReLU neural network with $d-1$ hidden layers and $d^2+4d-5$ nonzero weights.
	\end{lemma}

	\begin{proof}
		Recall that $\sigma(u)-\sigma(-u)=u$ for $u\in\mathbb{R}$. We prove the statement by induction on $d$. The case $d=2$ is easy because 
		$\min(x_1,x_2)=x_2-\sigma(x_2-x_1)=\sigma(x_2)-\sigma(-x_2)-\sigma(x_2-x_1)$ can be seen as a ReLU net with 1 hidden layer and 7 nonzero weights. We assume that $\min\{x_1,...,x_k\}$ can be seen as a ReLU net with $k-1$ hidden layers and $k^2+4k-5$ nonzero weights for any $2\le k\le d-1$. Then $\min(x_1,...,x_d)=\min(\min(x_1,...,x_{d-1}),x_d)=x_d-\sigma(x_d-\min\{x_1,...,x_{d-1}\})$, which has $1+(d-2)=d-1$ hidden layers since $\min\{x_1,...,x_{d-1}\}$ has $d-2$ hidden layers by our induction hypothesis. Let $J\in\mathbb{N}$, and
		denote $\mathcal{A}_{J}:\mathbb{R}\rightarrow\mathbb{R}^2$ by
		\begin{equation*}
			\mathcal{A}_J(x)=\sigma_{\mathbf{0}}W_J\cdots\sigma_{\mathbf{0}}W_2\sigma_{\mathbf{0}} W_1x,\quad x\in\mathbb{R},
		\end{equation*}
		where $W_1=(1,-1)'$, $W_j=I_2$ is the $2\times 2$ identity matrix for $j=2,\cdot\cdot\cdot,J$. Then $\mathcal{A}_J(x)=(\sigma(x),\sigma(-x))'$ for any $J\in \mathbb{N}$. It has $J$ hidden layers and $2J$ nonzero weights. Therefore we have
		\begin{equation*}
			\min\{x_1,...,x_d\}=W_1'\sigma_{\mathbf{0}}I_2\mathcal{A}_{d-2}(x_d)-\sigma(W_1'\mathcal{A}_{d-2}(x_d)-\min\{x_1,...,x_{d-1}\}),
		\end{equation*}
		which is a ReLU net with number of nonzero weights
		\begin{equation*}
			M_1+M_2+2||W_1||_0+||I_2||_0+1=d^2+4d-5,
		\end{equation*}
		where $M_1=2(d-2)$ is the number of nonzero weights of $\mathcal{A}_{d-2}$, and $M_2=(d-1)^2+4(d-1)-5$ is the number of nonzero weights of $\min\{x_1,...,x_{d-1}\}$ by our induction hypothesis. This completes the induction procedure and proves the lemma. 
	\end{proof}
	Figure \ref{min} illustrates how to express $\min\{x_1,x_2,x_3\}$ as a ReLU network.
		\begin{figure}[htp]
		{
			\centering
			\includegraphics[width=0.7\textwidth]{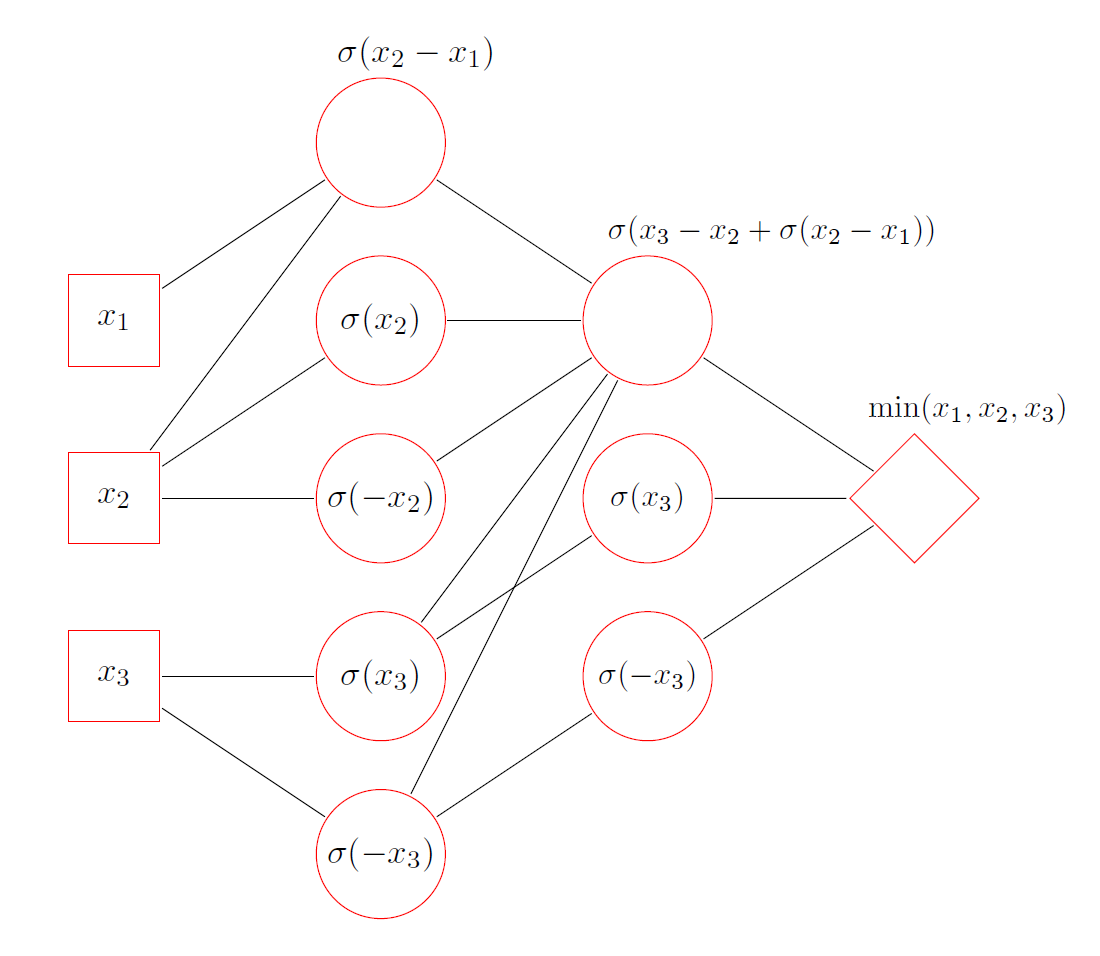}
			\caption{$\min\{x_1,x_2,x_3\}$ can be seen as a RuLU netwrok. This network has 16 connections and all threshold parameters are zero, hence it has 16 nonzero parameters. }
			\label{min}
		}
	\end{figure}
	
	\begin{proposition} \label{result3}
		Let $s,m,M\in\mathbb{N}$, $1\le p\le \infty$ and set $t=(2m+1)^s$. If $F:L^p([-1,1]^s)\rightarrow\mathbb{R}$ is a continuous functional with modulus of continuity $\omega_F$, then there exists a deep ReLU network $H:\mathbb{R}^t\rightarrow\mathbb{R}$ with depth $J=t^2+t+1$ and number of  nonzero weights $M$ such that
		\begin{equation*}
			\sup_{\mathbf{y}\in [-R,R]^t}|\mu_{m,F}(\mathbf{y})-H(\mathbf{y})|\le 2(2m+1)^s\omega_{F}\left(\frac{c_5t^{\frac{4}{t}}m^{\theta}}{M^{\frac{1}{t}}}\right),
		\end{equation*}
		where $c_5$ is a constant independent of $m, M$ or $\theta=2s|\frac{1}{p}-\frac{1}{2}|$.
	\end{proposition}
	
	\begin{proof}
		The proof can be divided into three steps.
		\begin{itemize}
			\item Step 1. We construct a continuous piecewise linear interpolation of $\mu_{m,F}$.
			
			The construction is motivated by \cite{Yarotsky18}. 	We consider the grid
			\begin{equation*}
				\mathcal{G}=\left\{-R+\frac{2R}{N}i:i=0,\cdot\cdot\cdot,N\right\}^t
			\end{equation*}
			on the cube $[-R,R]^t$, and we denote
			\begin{equation}\label{scaled_triangle}
				\triangle_{\mathbf{n},\rho}^{N}=\left\{\mathbf{y}\in\mathbb{R}^t:0\le y_{\rho(1)}-\frac{2Rn_{\rho(1)}}{N}\le\cdot\cdot\cdot\le y_{\rho(t)}-\frac{2Rn_{\rho(t)}}{N}\le\frac{2R}{N}\right\},
			\end{equation}
			for $\mathbf{n}=(n_1,\cdot\cdot\cdot,n_t)\in\mathbb{Z}^t$, $\mathbf{y}=(y_1,\cdot\cdot\cdot,y_t)$ and $\rho\in\mathcal{P}_t$. By scaling the grid and using Lemma \ref{lemma_partition} in Appendix, we know $\left\{\triangle^N_{\mathbf{n},\rho}\right\}_{\mathbf{n}\in\mathbb{Z}^t,\rho\in\mathcal{P}_t}$ is a partition of $\mathbb{R}^t$. Now we define the piecewise linear interpolant $H:\mathbb{R}^t\rightarrow\mathbb{R}$ as
			\begin{equation*}
				H(\mathbf{y})=\sum\limits_{\xi\in \mathcal{G}}\mu_{m,F}(\xi)\psi\left(\frac{N}{2R}(\mathbf{y}-\xi)\right).
			\end{equation*}
			According to properties (a), (b) and (c) satisfied by $\psi$, we know $H$ is continuous on $\mathbb{R}^t$, linear on every simplex $\triangle_{\mathbf{n},\rho}^N$ and interpolates $\mu_{m,F}$ at every $\xi\in\mathcal{G}$.
			
			\item Step 2. We show that $H$ can be viewed as a deep ReLU net and we calculate depth and number of nonzero weights of $H$.
			
			Recall  the expression \eqref{14} for $\psi$ and fix $\xi\in\mathcal{G}$.
			We concatenate elements in $\{1-\frac{N}{2R}\xi_k+\frac{N}{2R}y_k\}_{k=1}^t$, $\{1+\frac{N}{2R}\xi_k-\frac{N}{2R}y_k\}_{k=1}^t$ and $\{1+\frac{N}{2R}(y_k-y_j)\}_{k\neq j, k,j=1,...,t}$
				into a vector $\mathbf{a}=(a_1,...,a_{t^2+t})\in\mathbb{R}^{t^2+t}$, then by \eqref{14}, we know
				\begin{equation*}
					\begin{aligned}
						\psi\left(\frac{N}{2R}(\mathbf{y}-\xi)\right)&=\sigma\left(\min\{a_i:i=1,...,t^2+t\}\right)\\
						&=\sigma\left(\min\{\sigma(a_i):i=1,...,t^2+t\}\right).
					\end{aligned}
			\end{equation*}
			The last equality above is obtained by discussing the two cases of $\min_{i}\{a_i\}< 0$ and $\min_{i}\{a_i\}\ge 0$.
			According to Lemma \ref{nonzero}, we know $\min\{\sigma(a_i):i=1,...,t^2+t\}$ is actually a ReLU network
			with depth $t^2+t-1$ and number of nonzero weights  $(t^2+t)^2+4(t^2+t)-5$ when $(\sigma(a_i))_i$ is viewed as the input vector. Moreover, it requires $1$ hidden layer and $3t(t-1)+4t$ nonzero weights from $\mathbf{y}$ to $\sigma(\mathbf{a})$, therefore $H$ can be viewed as a ReLU network with depth $J=t^2+t-1+2=t^2+t+1$  and number of nonzero weights
			\begin{equation}\label{boundM}
				M\le c_3t^4(N+1)^t,
			\end{equation}
			for some absolute constant $c_3$. 
			
			\item Step 3. We approximate $\mu_{m,F}$ by $H$ and estimate the approximation error.
			
			For any $\mathbf{y}\in[-R,R]^t$, we know that there exists at least one simplex given in \eqref{scaled_triangle} that contains $\mathbf{y}$, and we denote this simplex by $\triangle$. Moreover, we denote by $\tilde{H}=H|_{\triangle}$ the restriction of $H$ to $\triangle$.
			The modulus of continuity of $\tilde{H}$ can be bounded by that of $\mu_{m,F}$ due to the piecewise linearity of $H$. Actually, if we use $\mu$ instead of $\mu_{m,F}$ for simplicity, then for $r>0$,
			\begin{equation*}
				\begin{aligned}
					\omega_{\tilde{H}}(r)&=\sup\big\{|\tilde{H}(\mathbf{y}_1)-\tilde{H}(\mathbf{y}_2)|:|\mathbf{y}_1-\mathbf{y}_2|_2\le r,\mathbf{y}_1,\mathbf{y}_2\in\triangle\big\}\\
					&=\sup\big\{|{\nabla \tilde{H}(\mathbf{y}_2)}^{T}(\mathbf{y}_1-\mathbf{y}_2)) |:|\mathbf{y}_1-\mathbf{y}_2|_2\le r,\mathbf{y}_1,\mathbf{y}_2\in\triangle\big\}\\
					&\le \sup_{\mathbf{y}_2\in\triangle}|\nabla \tilde{H}(\mathbf{y}_2)|_2r \le \sqrt{t}\sup_{\mathbf{y}_2\in\triangle}|\nabla \tilde{H}(\mathbf{y}_2)|_{\infty}r,
				\end{aligned}
			\end{equation*}
			where $\nabla=(\partial_1,...,\partial_t)$ is the gradient operator. Since $\tilde{H}$ is linear and coincides with $\mu$ on every node in $\triangle$, we have $\partial_k \tilde{H}(\mathbf{y}_2)=\frac{N}{2R}(\mu(\xi_k)-\mu(\eta_k))\le \frac{N}{2R}\omega_{\mu}\left(\frac{2R}{N}\right)$ for $k=1,...,t$,
			 where $\xi_k,\eta_k$ are the vertices of $\triangle$ having the same coordinates except the $k$-th coordinate, hence $|\xi_k-\eta_k|=\frac{2R}{N}$.  Therefore, we have
			\begin{equation} \label{omegeH}
				\omega_{\tilde{H}}(r)\le \frac{\sqrt{t}N}{2R} \omega_{\mu}\left(\frac{2R}{N}\right)r.
			\end{equation}
			
			Let $e(\mathbf{y})=\mu(\mathbf{y})-H(\mathbf{y})$, denote by $\mathbf{y}^*$ a nearest vertex of $\triangle$ to $\mathbf{y}$. Note that $e(\mathbf{y}^*)=0$ and $|\mathbf{y}-\mathbf{y}^*|\le \frac{\sqrt{t}R}{N}$, then
			\begin{equation*}
				\begin{aligned}
					|e(\mathbf{y})|&=|e(\mathbf{y})-e(\mathbf{y}^*)|\le \omega_e(|\mathbf{y}-\mathbf{y}^*|_2)\\
					&\le \omega_e\left(\frac{\sqrt{t}R}{N}\right)\le \omega_{\mu}\left(\frac{\sqrt{t}R}{N}\right)+\omega_{\tilde{H}}\left(\frac{\sqrt{t}R}{N}\right).
				\end{aligned}
			\end{equation*}
			Observe that the integer part $\lfloor \sqrt{t}\rfloor$ of $\sqrt{t}$ is no less then $\sqrt{t}/2$. Then by the sub-additivity of the modulus of continuity, we have
			\begin{equation*}
				\begin{aligned}
					\omega_{\mu}\left(\frac{\sqrt{t}R}{N}\right)&=\omega_{\mu}\left(\frac{\sqrt{t}}{2}\frac{2R}{N}\right)
					\le 	\omega_{\mu}\left(\lfloor \sqrt{t}\rfloor\frac{2R}{N}\right)\\
					&\le \lfloor\sqrt{t}\rfloor\omega_{\mu}\left(\frac{2R}{N}\right)\le \sqrt{t}\omega_{\mu}\left(\frac{2R}{N}\right).
				\end{aligned}
			\end{equation*}
			This together with (\ref{omegeH}) leads to
			\begin{equation*}
				\begin{aligned}
					\sup_{\mathbf{y}\in [-R,R]^t}|\mu_{m,F}(\mathbf{y})-H(\mathbf{y})|&\le
					\sqrt{t}\omega_{\mu}\left(\frac{2R}{N}\right)+\frac{t}{2}\omega_{\mu}\left(\frac{2R}{N}\right)\\
					&\le 2t\omega_{\mu}\left(\frac{2R}{N}\right)\\
					&\le 2(2m+1)^s\omega_{F}\left(\frac{c_2m^{\theta}}{N}\right),
				\end{aligned}
			\end{equation*}
			where the last inequality is due to Lemma \ref{omegamu}, $c_2=2c_1^2C_K$ is a constant and $\theta=2s|\frac{1}{p}-\frac{1}{2}|$. From (\ref{boundM}), we know $N\ge c_4M^{\frac{1}{t}}t^{-\frac{4}{t}}$ for some constant $c_4$, then
			\begin{equation*}
				\sup_{\mathbf{y}\in [-R,R]^t}|\mu_{m,F}(\mathbf{y})-H(\mathbf{y})|\le 2(2m+1)^s\omega_{F}\left(\frac{c_5t^{\frac{4}{t}}m^{\theta}}{M^{\frac{1}{t}}}\right)
			\end{equation*}
			with $c_5=c_2/c_4$. 
		\end{itemize}
		All the three steps lead to the desired conclusion. 
	\end{proof}
	Another natural approach to approximate the function $\mu_{m,F}$ on $[-R,R]^t$ is to first apply the classical result \cite[Theorem 3.2]{Mhaskar93} on approximation by a network $\tilde{F}$ of depth $\lceil\log t/\log 2\rceil+1$ induced by a sigmoid type activation function like $\sigma_2(u)=(\max\{u,0\})^2$ and then approximate $\sigma_2$ by a ReLU net. However, this approach raises a technical barrier: while $\sigma_2$ may be approximated to an arbitrary accuracy by a ReLU network of a fixed width, it can be approximated to an accuracy $\epsilon>0$ by a ReLU net only if the net has $O(\log \frac{1}{\epsilon})$ depth and $O(\frac{1}{\epsilon}\log \frac{1}{\epsilon})$ width, as shown in \cite{Yarotsky18}. The increasing depth and width as $\epsilon\rightarrow 0$ lead to more complicated estimates involving the parameter sizes of the network $\tilde{F}$ and much more work for getting rates of approximation as $t$ increases. It would be interesting to carry out the complete analysis for this approach.

	\subsection{Proof of main results}
	Now we are in the position to prove Theorems \ref{result1} and \ref{corro} stated in Section 3.
	\begin{proof}[Proof of Theorem \ref{result1}]
		From \eqref{phi}, we know $H(\phi(V_mf))$ can be seen as a parametric functional neural network in \eqref{parametric}. In this case $d_0=t$. Denote $\Theta_v(f):=H(\phi(V_mf)).$ The numerical weights in $\Theta_v$ are the same as those in $H$, therefore $M(\Theta_v)=M.$
		
		Combining Propositions \ref{result2} and \ref{result3}, we have
		\begin{equation*}
			\begin{aligned}
				\sup_{f\in K}|F(f)-\Theta_v(f)|&\le \sup_{f\in K}|F(f)-F(V_mf)|+\sup_{f\in K}|\mu_{m,F}(\phi(V_mf))-H(\phi(V_mf))|\\
				&\le  \sup_{f\in K}|F(f)-F(V_mf)|+ \sup_{\mathbf{y}\in [-R,R]^t}|\mu_{m,F}(\mathbf{y})-H(\mathbf{y})|\\
				&\le \omega_F(c\epsilon_{m,K}) +2(2m+1)^s\omega_{F}\left(\frac{c_5t^{\frac{4}{t}}m^{\theta}}{M^{\frac{1}{t}}}\right),
			\end{aligned}
		\end{equation*}
		which completes the proof of Theorem \ref{result1} with $C=\max\{c,c_5\}$.
	\end{proof}
	
	\begin{proof} [Proof of Theorem \ref{corro}]
		Since $K$ is a set of functions satisfying a H\"older condition of order $\beta$, we have $\epsilon_{m,K}=c_7m^{-\beta}$ for some constant $c_7$ as stated in Remark \ref{remark}. Then from Theorem \ref{result1} we have
		\begin{equation}\label{rate}
			\sup_{f\in K}|F(f)-\Theta_v(f)|\le c_6\left(Cc_7m^{-\beta}\right)^{\lambda}+2c_6(2m+1)^s\left(\frac{Cm^{\theta}t^{\frac{4}{t}}}{M^{\frac{1}{t}}}\right)^{\lambda}.
		\end{equation}
	Note that $(2m)^s\le t=(2m+1)^s\le (3m)^s$. We can simplify the bound in \eqref{rate} to the form
	\begin{equation*}
		\sup_{f\in K}|F(f)-\Theta_v(f)|\le c_8\left(m^{-\beta\lambda}+m^{s+\theta\lambda}\left(\frac{(3m)^{4s}}{M}\right)^{\frac{\lambda}{(2m)^s}}\right),
	\end{equation*}
	where $c_8=c_6(Cc_7)^{\lambda}+2\cdot3^sc_6(C)^{\lambda}$. To find a good choice of $m$, we try to balance the two terms of the above bound, and compare $m^{-(\beta\lambda+s+\theta\lambda)}$ with $\left(\frac{(3m)^{4s}}{M}\right)^{\frac{\lambda}{(2m)^s}}$. Taking logarithms yields $-(\beta\lambda+s+\theta\lambda)\log m$ and $\frac{\lambda}{(2m)^s}\left(4s\log(3m)-\log M\right)$. When $4s\log(3m)\ll\log M$, we can compare $-(\beta\lambda+s+\theta\lambda)\log m$ with $-\frac{\lambda}{(2m)^s}\log M$. Therefore, we choose $m$ to be the integer such that
	\begin{equation}\label{zhou_revise}
		c_9m^s\log (3m)\le \log M <c_9(m+1)^s\log (3(m+1)),
	\end{equation}
	where  $c_9=\left(8s+2^{1+s}(s/\lambda+\theta+\beta)\right)$.
	This integer exists when $\log M\ge 2c_9\log 3$. Under this choice, $4s\log(3m)\le \frac{1}{2}\log M$ and thereby, $\frac{\lambda}{(2m)^s}\left(4s\log(3m)-\log M\right)\le -\frac{1}{2}\frac{\lambda}{(2m)^s} \log M$, which implies
	\begin{equation*}
		\begin{aligned}
			m^{s+\theta\lambda}\left(\frac{(3m)^{4s}}{M}\right)^{\frac{\lambda}{(2m)^s}}&\le \exp\left\{(s+\theta\lambda)\log m-\frac{1}{2}\frac{\lambda}{(2m)^s}\log M\right\}\\
			&\le \exp\left\{-\beta\lambda\log m\right\}= m^{-\beta\lambda}.
		\end{aligned}
	\end{equation*}
	It follows that
	\begin{equation*}
		\sup_{f\in K}|F(f)-\Theta_v(f)|\le 2c_8m^{-\beta\lambda}. 
	\end{equation*}
	From \eqref{zhou_revise}, we find $m^s\le \frac{\log M}{c_9\log(3m)}\le \frac{\log M}{c_9\log 3}\le \log M$, which implies $\log (3(m+1))\le 6\log(m)\le \frac{6}{s}\log(\log M)$. It then follows  that
	\begin{equation*}
		\begin{aligned}
			\log M&<c_9(m+1)^s\log (3(m+1))\\
			&\le \frac{6c_9}{s}(m+1)^s\log(\log M)\\
			&\le  \frac{6c_92^s}{s}m^s\log(\log M).
		\end{aligned}
	\end{equation*}
	Therefore
	\begin{equation*}
			\sup_{f\in K}|F(f)-\Theta_v(f)|\le 2c_8\left(\frac{s}{6c_9 2^s}\right)^{-\frac{\beta\lambda}{s}}\left(\frac{\log M}{\log (\log M)}\right)^{-\frac{\beta\lambda}{s}}.
	\end{equation*}
	Again from \eqref{zhou_revise}, we know $ \log M <c_9(m+1)^s\log (3(m+1))\le 3c_9(m+1)^{s+1}$, which implies $\log (3m)>\log(m+1)>\frac{\log (\log M/(3c_9))}{s+1}\ge \frac{\log (\sqrt{\log M})}{s+1}=\frac{\log \log (M)}{2(s+1)}$ as long as $\log M\ge (3c_9)^2$. It then follows that
	\begin{equation*}
		\log M\ge c_9 m^s\log (3m)>\frac{c_9}{2(s+1)}m^s\log (\log M).
	\end{equation*}
	Then the depth is
	\begin{equation*}
		\begin{aligned}
			J&=(2m+1)^{2s}+(2m+1)^s+1\\
			&\le (3m)^{2s}+(3m)^{s}+1< 3(3m)^{2s}\\
			&\le 3^{2s+1}\frac{4(s+1)^2}{c_9^2}\left(\frac{\log M}{\log (\log M)}\right)^2.
		\end{aligned}
	\end{equation*}
	This proves the desired conclusion when $\log M\ge \max\{2c_9\log 3,(3c_9)^2\}:=c_{10}$. When $\log M<c_{10}$, we take $m=1$ and the direct statement is also seen. This proves Theorem \ref{corro}.
	\end{proof}
	
	\section{Discussion}
	In this paper, we construct a parametric functional deep ReLU network $\Theta_v$ by the piecewise linear interpolation under a simple triangulation. We derive a rate of approximation in Theorem \ref{result1} in terms of the degree of polynomial approximation of $f$ and the total number of nonzero weights in the network $\Theta_v$. Theorem \ref{corro} shows that the proposed parametric functional deep ReLU network can achieve almost the same rate as shallow functional neural networks in \cite{Mhaskar97} with infinitely differentiable activation functions satisfying the assumption \eqref{assump1}. It would be interesting to extend our study to the approximation of nonlinear functionals with structures by structured functional neural networks such as convolutional neural networks \cite{Zhou20,Feng,Mao} and applications in generalization analysis and some practical applications \cite{Feng2022}.
	

	\section{Appendix}
	This appendix contains some lemmas which follow immediately from some known results in the literature. For completeness, we provide proofs.

	\begin{lemma}\label{lemma_partition}
		Let $d\in\mathbb{N}$, we denote by $\triangle_{\mathbf{n},\rho}$ a simplex in $\mathbb{R}^{d}$ defined by
		\begin{equation*}
			\triangle_{\mathbf{n},\rho}=\big\{\mathbf{x}\in\mathbb{R}^d:0\le x_{\rho(1)}-n_{\rho(1)}\le\cdot\cdot\cdot\le x_{\rho(d)}-n_{\rho(d)}\le 1\big\},
		\end{equation*}
		where $\mathbf{n}=(n_1,\cdot\cdot\cdot,n_d)\in\mathbb{Z}^d$ and $\rho$ is a permutation of $d$ elements. Let $\mathcal{P}_d$ be the set of all permutations of $d$ elements, then the set of all simplexes $\left\{\triangle_{\mathbf{n},\rho}\right\}_{\mathbf{n}\in\mathbb{Z}^d,\rho\in\mathcal{P}_d}$ is a partition of $\mathbb{R}^d$.
	\end{lemma}
	\begin{proof}
		We denote $\mathbb{I}_\mathbf{n}=[n_1,1+n_1]\times\cdots\times[n_d,1+n_d]$, which is a unit cube in $\mathbb{R}^d$ with $\mathbf{n}$ being one of its vertices. First, it is obvious that
		\begin{equation*}
			\triangle_{\mathbf{n},\rho_1}\cap\triangle_{\mathbf{n},\rho_2}=\emptyset, \quad \forall\rho_1\neq\rho_2,
		\end{equation*}
		and note that for any $\mathbf{x}\in\mathbb{I}_\mathbf{n}$, there exists a permutation $\rho$ such that $\{x_1-n_1,...,x_d-n_d\}$ can be ordered in the following way
		\begin{equation*}
			x_{\rho(1)}-n_{\rho(1)}\le \cdots\le	x_{\rho(d)}-n_{\rho(d)},
		\end{equation*}
		therefore $\{\triangle_{\mathbf{n},\rho}\}_{\rho\in\mathcal{P}_d}$ is a partition of $\mathbb{I}_\mathbf{n}$. Moreover, since $\{\mathbb{I}_\mathbf{n}\}_{\mathbf{n}\in\mathbb{Z}^d}$ is partition of $\mathbb{R}^d$, we know $\left\{\triangle_{\mathbf{n},\rho}\right\}_{\mathbf{n}\in\mathbb{Z}^d,\rho\in\mathcal{P}_d}$ is a partition of $\mathbb{R}^d$.
	\end{proof}
	
	\begin{lemma}\label{lemma_convex}
		Let $d\in\mathbb{N}$ and denote $S_0$ to be the union of all simplexes defined in Lemma \ref{lemma_partition} that contain $\mathbf{0}\in\mathbb{R}^d$, that is,
		\begin{equation*}
			S_0=\bigcup_{\mathbf{0}\in\triangle_{\mathbf{n},\rho}}\triangle_{\mathbf{n},\rho},
		\end{equation*}
		then $S_0$ is a convex set.
	\end{lemma}
	\begin{proof}
		Denote
		\begin{equation*}
			S'=\bigcap_{\substack{k,l=1,...,d\\ l\neq k}} \left\{\mathbf{x}\in\mathbb{R}^d:1+x_k\ge0,1-x_k\ge0,1+x_l\ge x_k\right\}.
		\end{equation*}
		Notice that $S'$ is a convex set since the intersection of half-spaces is convex, it suffices to show $S_0=S'$ in the rest of the proof.
		
		Firstly, we show that $S_0\subset S'$. For any simplex $\triangle_{\mathbf{n},\rho}$ that contains $\mathbf{0}$, we have
		\begin{equation*}
			0\le -n_{\rho(1)}\le\cdot\cdot\cdot\le -n_{\rho(d)}\le 1,
		\end{equation*}
		which implies $\mathbf{n}\in\{-1,0\}^d$ and if $n_{\rho(j_0)}=-1$ for some $j_0$, then $n_{\rho(j)}=-1$ for $j>j_0$. 
		\begin{itemize}
			\item When $\mathbf{n}=\mathbf{0}$, we have $\triangle_{\mathbf{n},\rho}=\big\{\mathbf{x}\in\mathbb{R}^d:0\le x_{\rho(1)}\le\cdot\cdot\cdot\le x_{\rho(d)}\le 1\big\}$. For any  $\mathbf{x}\in\triangle_{\mathbf{n},\rho}$, we have  $|x_k|\le 1$ for all $k$ and $x_k\le 1+x_l$ for any $k\neq l$, which implies $\mathbf{x}\in S'$, that is, $\triangle_{\mathbf{n},\rho}\subset S'$;
			\item When $\mathbf{n}\neq \mathbf{0}$, we denote $j_*$ to be smallest integer such that $n_{\rho(j_*)}=-1$. For any $\mathbf{x}\in\triangle_{\mathbf{n},\rho}$, we have 
			\begin{equation*}
				0\le x_{\rho(1)}\le\cdots\le x_{\rho(j_*-1)}\le x_{\rho(j_*)}+1\le\cdots\le x_{\rho(d)}+1\le 1,
			\end{equation*}
			which implies $|x_k|\le 1$ for all $k$ and $x_k\le 1+x_l$ for any $k\neq l$, that is $\triangle_{\mathbf{n},\rho}\subset S'$.
		\end{itemize}
		Therefore we have $\triangle_{\mathbf{n},\rho}\subset S'$ for any $\triangle_{\mathbf{n},\rho}$ containing $\mathbf{0}$, which implies $S_0\subset S'$. 
		
		Secondly, we show that $S'\subset S_0$. Let $\mathbf{x}\in S'$, then $x_k\in [-1,1]$ for $k=1,...,d$. We define $\mathbf{n}=(n_1,...,n_d)$ by 
		\begin{equation*}
			n_k=\left\{
			\begin{aligned}
				0,&\quad \text{if } x_k\in [0,1], \\
				-1,&\quad \text{if } x_k\in [-1,0),
			\end{aligned}
			\right.\quad k=1,...,d.
		\end{equation*} 
		Now we have $0\le x_k-n_k\le 1$ for all $k$. Moreover, the elements in the $\{x_1-n_1,\cdots,x_d-n_d\}$ can be order as below
		\begin{equation*}
			0\le x_{\rho(1)}-n_{\rho(1)}\le\cdots\le x_{\rho(d)}-n_{\rho(d)}\le 1,
		\end{equation*}
		with some permutation $\rho$, which means $\mathbf{x}\in \triangle_{\mathbf{n},\rho}$. Besides, we know $\mathbf{0}\in\triangle_{\mathbf{n},\rho}$ from the definition of $\mathbf{n}$, therefore $S'\subset S_0$, and the proof is completed.
	\end{proof}
	
	\begin{lemma}\label{lemma_affine}
		Let $d\in\mathbb{N}$, $\triangle_{\mathbf{n},\rho}$ be a simplex defined in Lemma \ref{lemma_partition} having $\mathbf{0}\in\mathbb{R}^d$ as a vertex, and denote $U$ to be the set of all vertices of $\triangle_{\mathbf{n},\rho}$. If
		$h:\mathbb{R}^d\rightarrow\mathbb{R}$ is a linear function satisfying that $h(\mathbf{0})=1$, and $h(\mathbf{x})=0$ for $\mathbf{x}\in U\setminus\{\mathbf{0}\}$, then
		\begin{equation*}
			h\in\left\{\mathbf{x}\mapsto 1+ax_l-bx_k :(a,b)\in\{(1,0),(1,1),(-1,0)\},k,l\in\{1,...,d\},k\neq l\right\}.
		\end{equation*}
	\end{lemma}
	\begin{proof}
		Note that the simplex $\triangle_{\mathbf{n},\rho}$ can be viewed as the intersection of $d+1$ closed half-spaces,
		\begin{equation*}
			\begin{aligned}
				\triangle_{\mathbf{n},\rho}=&\cap_{j=1}^{d-1}\{\mathbf{x}:x_{\rho(j)}-n_{\rho(j)}\le x_{\rho(j+1)}-n_{\rho(j+1)}\}\cap\{\mathbf{x}:n_{\rho(1)}\le x_{\rho(1)}\}\\
				&\cap\{\mathbf{x}:x_{\rho(d)}\le n_{\rho(d)}+1\}.
			\end{aligned}
		\end{equation*} 
		If we denote $\rho(\mathbf{y})=(y_{\rho(1)},...,y_{\rho(d)})$ for any vector $\mathbf{y}\in\mathbb{R}^d$, then the set of all vertices $U$ can be represented by
		\begin{equation*}
			U=\left\{u_i\in\mathbb{R}^d:\rho(u_i)=\rho(\mathbf{n})+e_i:i=0,1,...,d\right\},
		\end{equation*}
		where $e_i$ is the vector such that the last $i$ entries are $1$, and the rest of entries are $0$. Since $h$ is linear, we can write $h$ as the following form
		\begin{equation*}
			h(\mathbf{x})=\langle \mathbf{a},\rho(\mathbf{x})-\rho(\mathbf{n})\rangle+a_0,
		\end{equation*}
		for some $\mathbf{a}=(a_1,...,a_d)\in\mathbb{R}^d$, $a_0\in\mathbb{R}$. Since $\mathbf{0}$ is in $U$, there must exist $k\in \{0,...,d\}$ such that $\rho(\mathbf{n})+e_k=\mathbf{0}$. The proof is completed by directly solving the linear system with $d+1$ variables $(a_0,...,a_d)$, 
		\begin{equation*}
			h(u_i)=\langle \mathbf{a},e_i\rangle+a_0=\delta_{ik},j=0,...,d,
		\end{equation*}
		where $\delta_{ik}=1$ if $i=k$, and $\delta_{ik}=0$ if $i\neq k$.
	\end{proof}
	
	\section*{Acknowledgments} The second author is supported partially by the Research Grants Council of Hong Kong [Project No. HKBU 12302819] and National Natural Science Foundation of China [Project No. 11801478]. The third author is supported partially by National Natural Science Foundation of China  [Project No. 11971048] and NSAF [Project No. U1830107]. The last author was supported partially by the Research Grants Council of Hong Kong [Project Numbers CityU 11308121, N\_CityU102/20, C1013-21GF], Laboratory for AI-powered Financial Technologies, Hong Kong Institute for Data Science, and National Natural Science Foundation of China [Project No. 12061160462]. This paper in its first version was written when the last author worked at City University of Hong Kong and visited SAMSI/Duke during his sabbatical leave. He would like to express his gratitude to their hospitality and financial support. The corresponding author is Jun Fan.

\end{document}